\documentclass[letterpaper]{article} 
\usepackage{aaai23}  
\usepackage{times}  
\usepackage{helvet}  
\usepackage{courier}  
\usepackage[hyphens]{url}  
\usepackage{graphicx} 
\usepackage{float}
\usepackage{natbib}  
\usepackage{caption} 
\usepackage{algorithm}
\usepackage{algorithmic}
\usepackage{amsmath}
\usepackage{amsthm}
\usepackage{xcolor}
\newtheorem{lem}{Lemma}
\newtheorem{theorem}{Theorem}
\newtheorem{corollary}{Corollary}
\newtheorem{definition}{Definition}

\usepackage{newfloat}
\usepackage{listings}
\DeclareCaptionStyle{ruled}{labelfont=normalfont,labelsep=colon,strut=off} 
\floatstyle{ruled}
\newfloat{listing}{tb}{lst}{}
\floatname{listing}{Listing}
\title{Max Markov Chain}
\author {
    Yu Zhang, 
    Mitchell Bucklew
}
\affiliations {
    School of Computing and Augmented Intelligence \\
    Arizona State University,  Tempe, AZ 85281 US \\ 
    \{\it{yzhan442, mbucklew}\}@asu.edu
}

\usepackage{bibentry}

\begin{document}

\maketitle

\begin{abstract}
In this paper, we introduce Max Markov Chain (MMC), a novel representation for a useful subset of High-order Markov Chains (HMCs) with sparse correlations among the states. 
MMC is parsimony while retaining the expressiveness of HMCs. 
Even though parameter optimization is generally intractable as with HMC approximate models, it has an analytical solution, better sample efficiency, and the desired spatial and computational advantages over HMCs and approximate HMCs. 
Simultaneously, efficient approximate solutions exist for this type of chains as we show empirically, which allow MMCs to scale to large domains where HMCs and approximate HMCs would struggle to perform. 
We compare MMC with HMC, first-order Markov chain, and an approximate HMC model in synthetic domains with various data types to demonstrate that MMC is a valuable alternative for modeling stochastic processes and has many potential applications. 
\end{abstract}


\author{Yu Zhang and Mitchell Bucklew}



\section{Introduction}

A Markov Chain (MC) is a simple but powerful tool for modeling stochastic processes.
MCs assume that the current state is independent of all ancestral states given a fixed number of immediate parental and ancestral states (determined by the order of the chain). 
The simplicity of MC makes it  a desired choice for modeling various stochastic processes~\cite{cowles1996markov,elfeki2001markov, ye2004robustness}. 
For domains with long-term dependencies, High-order Markov Chains (HMCs) are often needed.
However, HMCs are generally expensive to maintain due to its significant parameter size, 
which is exponential in the order of the chain. 
This also makes learning HMCs extremely sample inefficient. 
Even though approximate models have been studied~\cite{raftery1994estimation, berchtold1995autoregressive, berchtold2002mixture}, 
they often suffer from intense computational needs due to the inherent complexity in parameter optimization. 

\begin{figure}
  \centering
  \includegraphics[width=1\linewidth]{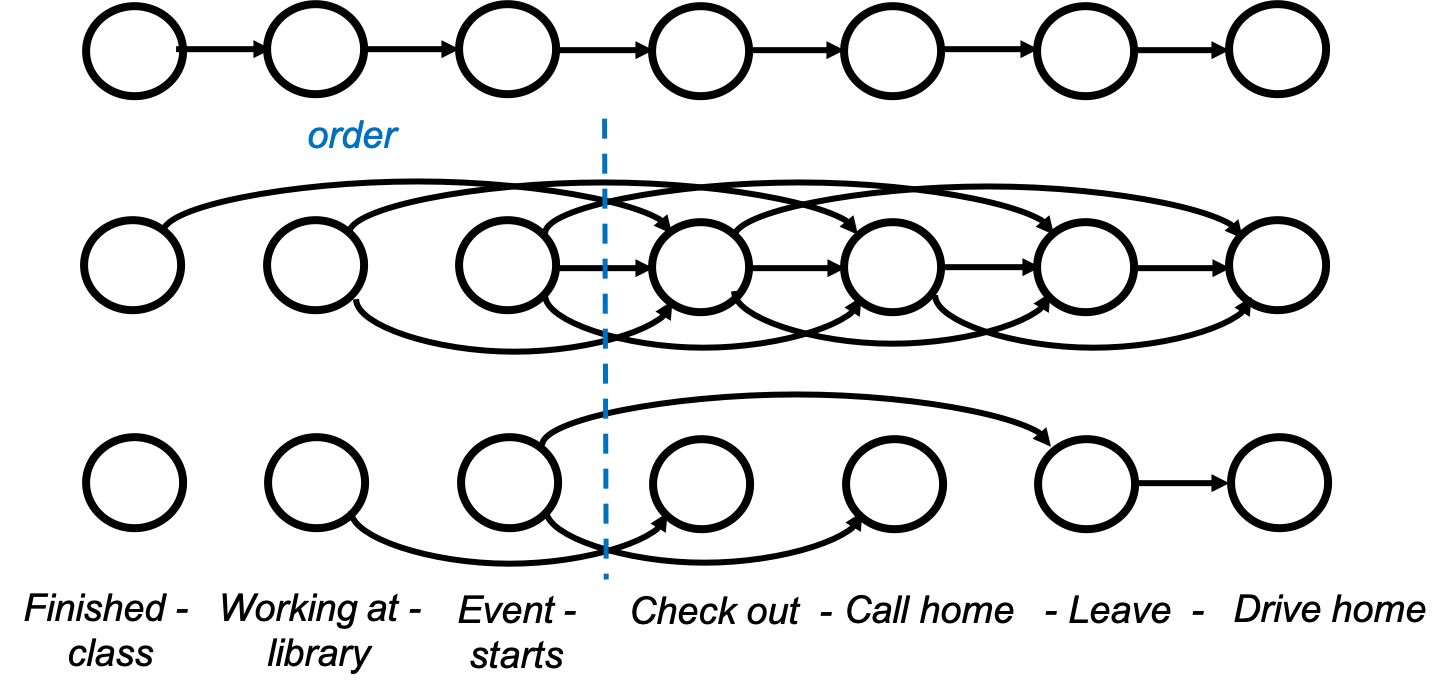}
  \caption{Graphical representations of, from top to bottom, First-order Markov Chain, High-order Markov Chain, and Max Markov Chain. The order is $3$ for both MMC and HMC.}
  \label{fig:models}
  \vskip-10pt
\end{figure}

We focus on discrete time-homogeneous HMCs~\cite{parzen1999stochastic} where the transition probabilities do not change over time. 
Instead of designing approximate models in the traditional way with the goal of approximating {\it any} HMC process, 
our aim here is to consider a subset of HMC processes that are useful and simultaneously easier to model.
Note that although approximate HMC models (e.g., \cite{raftery1994estimation}) are necessarily restricted in a similar way, 
less is understood about the expressiveness of such models without empirical analyses and careful interpretation of the results, i.e., what are the domains that can (cannot) be modeled well by such models. 
Furthermore, parameter optimization for such approximate models is often a challenge, rendering them impractical for many medium to large domains.  

In this paper, we introduce the Max Markov Chain (MMC) for modeling HMC processes where there exist only sparse correlations among the states.
See Fig. \ref{fig:models} for an illustration of the differences between MMC and other common MC models. 
Our target domains are those where a state can ``{\it generate}' a state according to a distribution in any future step within the order of the MMC. 
It is referred to as a max model since these {\it generating states} would also be competing with each other for state generation: only a single previous state is assumed to have generated the current state under consideration.
Such an assumption significantly simplifies the structure of the chain as seen in Fig. \ref{fig:models}.
Consider the following motivating example where MMC would be appropriate:

John arrived at the library and is working on his class assignments in the library. 
An event suddenly kicks off in the library which is disrupting to John. 
However, before John leaves, 
he will need to check out a few textbooks to continue working on the assignments at home.
Since he will arrive at home earlier today, 
he may make a quick call before leaving and driving home. 
The correlations among these events are shown in Fig. \ref{fig:models}, which
can be modeled by an MMC. 
Modeling this scenario with HMC is possible but unnecessary while  a first-order Markov chain would be insufficient given the long-term dependencies. 
Certain types of human behaviors described as both reactive and  deliberative may fit the MMC model assumptions well~\cite{schmidt2000modelling}.
Such sparse correlations between the states in a process is what we strive to model in this work.
An intuitive way to think about an MMC process is to consider a transition system for which any transition may be delayed by an unspecified but bounded amount of time.
Note that such a feature of MMC differs from jumping chains~\cite{metzner2009transition} and options~\cite{sutton1999between}, which are about the duration of transitions. 


We formally introduce MMC and provide an analytical solution for parameter estimation based on maximum likelihood. 
However, such a solution is generally intractable unless the state space is small. 
We provide two alternatives that aim at sub-optimal solutions, based on hill climbing and a greedy heuristic, respectively. 
In our evaluation, we compare MMC with three other Markov models with different data types to demonstrate the features of MMC, such as its scalability, sample efficiency, and superior performance for domains that satisfy or approximately satisfy the MMC model assumptions. 
Results confirm that MMC is a valuable alternative for modeling stochastic processes. 

\section{Related Work}

Markov chain was introduced to  model stochastic processes,
which has been applied to a diverse set of domains including physics~\cite{randall2006rapidly}, computer science~\cite{stewart1994introduction}, geography~\cite{chin1977modeling}, behavior and social sciences~\cite{benjamin1979use}.
Markov chains rely on the Markov assumption to simplify modeling, learning, and inference. 
Markov models have also been generalized to enable more expressiveness and model complexities, such as
the Hidden Markov Model (HMM)~\cite{baum1966statistical, fine1998hierarchical} and factored models~\cite{kearns1999efficient}. 
However, Markov chains require an exponential number of parameters in the order of the chain, which makes them intractable to maintain for complex domains. 
This also makes learning very sample inefficient. 
Approximate HMCs that are more parsimonious and quicker to learn are desired. 
Popular approximate models~\cite{jacobs1978discrete, raftery1985model} use auxiliary variables 
to combine the influences from each of the previous lags for generating the next state.
Such models have also been extended to consider model mixtures~\cite{berchtold2002mixture} where the influences to combine are specified with respect to one or multiple lags 
for better approximations. 

There are two main limitations of the existing approximate HMC models. 
While they are more parsimonious than HMCs,
learning to optimize the parameters is computationally challenging,
often through complex numerical procedures~\cite{raftery1985model, berchtold2002mixture}.  
Second, 
learning these models are still sample inefficient. 
This is mainly due to the fact that
they do not impose parsimony in {\it model structures}, 
which may introduce overfitting for domains with sparse correlations among the states.  
This is because, as shown in Fig. \ref{fig:models}, not all connections among the states are necessary. 
Although there are general solutions for addressing overfitting~\cite{ying2019overview} such as parameter regularization,
the fundamental problem remains. 

Max Markov Chain (MMC) addresses these issues by imposing model parsimony while retaining the ability to model long-term dependencies. 
As we will show later in the discussion, the assumption made in MMC can be gradually relaxed to converge to the full HMC model, resulting in a spectrum of models that are increasing in model complexity. 
The model structure of MMC may appear similar to skip-chain sequence models~\cite{galley2006skip} and variable-order Markov chains~\cite{roucos1982variable}. 
However, in these prior models, the skipping structures are assumed to be provided a priori or must be learned in a very expensive process. 
Finally, one may view the sparse correlations among the states as discovering causal relationships~\cite{pearl2003statistics}.

\section{Approach}

Next, we first provide some background for high-order Markov chains and a popular approximate model. 
We then introduce Max Markov Chain along with methods for parameter optimization given training data.
We focus on complete data. 
Parameter optimization under partial observable data will be address in future work. 

\subsection{Preliminaries}
A discrete time-homogeneous order Markov chain of order $K$ is specified as the probability distributions of the next state given the previous $K$ states (often referred to as {\it lags}):
\begin{equation}
    P(S_{t+K} | S_{t:t+K-1})
\label{eq:hmc-data}
\end{equation}

The parameter size of this model is $M^K(M - 1)$,
where $M$ is the size of the state space. 
The MTD model~\cite{raftery1985model}, an approximate HMC model, is specified as follows:

\begin{equation}
    P(S_{t+K} | S_{t:t+K-1}) = \Sigma_l \lambda_l q_{(t+l)(t+K)}
\end{equation}
where $q_{(t+l)(t+K)}$ is a value in an $M \times M$ transition matrix $Q$,
capturing the influence from state $S_{t+l}$ to $S_{t+K}$; $\lambda_l$ is a weight parameter associated with lag $l$ and satisfies $\sum_l \lambda_l = 1$.
Hence, an MTD model's parameter size is
$M(M - 1) + (K - 1)$.
Even though the MTD model is parsimonious in the parameter size, its model structure remains the same as that of HMC. It makes the assumption that the same state at the same lag contributes the same influence to $S_{t+K}$, regardless of the other states.

Our motivation for MMC is to address the limitations of the above models by further imposing parsimony on the model structure to expedite learning and make it scalable to large domains.
For the simplest MMC models that we study here, 
we make the assumption of sparse correlations such that only one state ($S_{t+l}$) in the lags is allowed to influence the current state ($S_{t+K}$). 
However, MMC allows the influence from $S_{t+l}$ to $S_{t+K}$ to be affected by the presence of the other states, in contrast to the basic MTD models. 
Imposing sparse model structure would necessarily introduce biases
when the modeling assumptions do not hold.
However, we argue that the advantages can outweigh the limitations in domains when the MMC assumptions hold or approximately hold. 
In the latter case,
MMC is designed to recover the most prominent correlations among the state pairs,
making it more sample efficient than MTD models. 



\subsection{Max Markov Chain}
Our innovation here is Max Markov Chain (MMC) that captures a useful subset of high-order Markov processes. 
In particular, an MMC is specified as follows:
\begin{align}
     P(S_{t+K} | S_{t:t+K-1}) = P(S_{t+K}| S_{t+l^*})
\label{eq:max}
\end{align}
where $l^* = argmax_{S_{t+K}, l} P(S_{t+K} | S_{t+l})$.

Intuitively, the generation of the next state is dominated by one of the previous states (e.g., $S_{t+l^*}$) with the maximum influence, as measured by the maximum transition  probability to any state.
To distinguish MMC from other MCs, we henceforth refer to the transition probabilities in MMC as the {\it generation probabilities} and the transition distribution for each state as its {\it generation distribution}. 

Essentially, the maximum generation probability for each state determines its priority for state generation. 
And the generation distribution of the new state depends only on the state with the highest priority in the lags. 
Theoretically, such a specification still represents a K-order Markov chain
since the generation distribution of $S_{t+K}$ depends on all the $K$ lags, albeit in a restrictive way (i.e., winner-take-all). 
Even though MMC only uses the same amount of parameters as a first-order Markov chain (i.e., $M(M - 1)$), 
such a specification allows it
to encode a subset of high-order Markov processes to extend its capability to modeling long-term dependencies, however, at the cost of forgetting  adjacency information as captured by first-order chains. 
In other words, MMC does not care about which lag actually generates the data. 
Consequently, while MTD model is a strict generalization of first-order Markov chain since it can perfectly fit  a first-order chain, the same cannot be said about MMC. 


\vskip-5pt
\begin{table}[h!]
  \begin{center}
    \label{tab:table1}
    \begin{tabular}{l| l|c|r} 
      $\rightarrow$& \textbf{$s_1$} & \textbf{$s_2$} & \textbf{$s_3$}\\
      \hline
      $s_1$ & 0 & 2 & 2\\
      $s_2$ & 2 & 0 & 0\\
      $s_3$ & 1 & 0 & 4\\
    \end{tabular}
    \caption{A data generation table for an MMC of order $2$ with $3$ states based on the training data  $3 3 3 3 3 3 1 2 3 1 2 3 1 $ (s is abbreviated for clarity), 
    assuming an SGO $s_1 \succ s_2 \succ s_3$.
    For every two consecutive states, the generation relationship for the next state is determined by which state among the two has a higher priority in the SGO. Each row indicates a generating state and the following numbers are the counts of a particular state generated by the generating state.}
    \vskip-10pt
    \label{tab:example}
  \end{center}
\end{table}



\subsection{Parameter Estimation}

Given training data $\mathcal{D}$ in the form of $d: s_{t:t+K-1} \rightarrow s_{t+K}$ $(d \in \mathcal{D})$, the learning problem can be described as maximizing the likelihood of the data under the i.i.d. assumption:

\begin{equation}
     \max_{\mathcal{M}} P(\mathcal{D} | \mathcal{M}) = \max_{\mathcal{M}} \Pi_d P(d | \mathcal{M})
    \label{eq:obj}
\end{equation}
where $\mathcal{M}$ represents the space of MMCs. 
Directly optimizing the above is infeasible since there is an infinite number of MMCs. 
To discuss our solution,  
first, we introduce {\it State Generation Order} (SGO):

\begin{definition}[State Generation Order (SGO)]
A state generation order specifies the order of priorities  over the state space for state generation according to Eq.
\eqref{eq:max}.
\end{definition}

We use $s_a \succ s_b$ to denote that state $s_a$ has a higher priority than $s_b$ for state generation.
This requires $p_a^* \geq p_b^*$ to hold where $p_a^*$ ($p_b^*$) represents the maximum generation probability of $s_a$ ($s_b$) for any state.
It is clear that an MMC introduces a unique SGO, when assuming that ties are broken consistently. 
We optimize Eq. \eqref{eq:obj} according to a 2-step process.
In the first step, we determine the SGO of the state space. 
In the second step, we optimize the parameters according to the SGO. 
We will show next that optimizing the parameters with a given SGO is not difficult. 
In such a situation, optimizing Eq. \eqref{eq:obj} boils down to iterating through the set of SGOs, which is finite, to search for the optimal SGO. 
The MMC that maximizes Eq. \eqref{eq:obj} is then the optimal SGO under the optimized parameters.

\subsubsection{Parameter Optimization under a Specified SGO}
Next, denoting the space of SGOs as $\mathcal{O}$, 
we consider optimizing for $P(\mathcal{D} | \mathcal{M}_o)$, where 
$o \in \mathcal{O}$ and $\mathcal{M}_o$ denotes the set of MMCs satisfying the given SGO, $o$.
Under the assumption that the data generation model is an MMC, when an SGO is given, the generation relationship among the data is also known (i.e., which state among $s_{t:t+K-1}$ generated $s_{t+K}$).
Refer to Table \ref{tab:example} for an example of how the generation relationship is determined. 
We denote the data entries that are generated by state $s$ in the training data $\mathcal{D}$ as $\mathcal{D}_s$, in the form of $s_{t+l} \rightarrow s_{t+K}$.
Each row in Table \ref{tab:example} corresponds to a $\mathcal{D}_s$ for the corresponding generating state $s$.
Then, the likelihood of $\mathcal{D}_s$ can be expressed by:
\begin{equation}
    P(\mathcal{D}_s | \mathcal{M}) = P(\mathcal{D}_s | p_1:p_m \text{ for state } s) = p_1^{n_1}p_2^{n_2},...,p_m^{n_m}
\label{eq:product}
\end{equation}
where $p_i$ is the generation probability of state $s_i$ for $s$ and $n_i$ is the number of times in $\mathcal{D}_s$ that $s_i$ is generated by $s$ (see Table \ref{tab:example}).
The $\{p\}$'s specify a distribution. 
Note that we intentionally refrain from specifying to which state these parameters (i.e., $n_1:n_m$ and $p_1:p_m$) belong to avoid clutter in the notation.  
Each state is associated with a separate set of these parameters. 
By taking the log of the likelihood, this transforms Eq. \eqref{eq:obj} to:
\begin{equation}
    \max_{\mathcal{M}} \log P(\mathcal{D} | \mathcal{M}) = \max_{\mathcal{M}} \Sigma_s \log P(\mathcal{D}_s | \mathcal{M})
\end{equation}

Given that the $\{p\}$'s specify a distribution,
$P(\mathcal{D}_s | \mathcal{M})$ in Eq. \eqref{eq:product} takes the unique maximum value when the $p$ values are aligned with the data, such that $p_i = \frac{n_i}{\sum_i n_i}$ (referred to as the {\it optimal value setting}).
The minimum value is at when one or more of the $p$ values are $0$.
Furthermore, $P(\mathcal{D}_s | \mathcal{M})$ monotonically decreases
as the $\{p\}$'s deviate from the optimal value setting.
Next, we introduce a few theoretical results that inform the parameter optimization process to be discussed later. 

\begin{lem}
Given an SGO that is respected by the data, meaning that the state priorities according to the generation distributions estimated from data, i.e., $p_i = \frac{n_i}{\sum_i n_i}$ align with the SGO, 
the maximum data likelihood $P(\mathcal{D} | \mathcal{M})$ is achieved by setting the MMC generation probabilities according to $p_i = \frac{n_i}{\sum_i n_i}$ for each state, respectively.
\label{lem:maxp}
\end{lem}

\begin{proof}
The proof is straightforward given that $P(\mathcal{D} | \mathcal{M}) = \Pi_s P(\mathcal{D}_s | \mathcal{M})$. Since the $\mathcal{D}_s$'s are fully specified by the SGO, 
the generation probabilities for each state can be set according to the data (i.e., $p_i = \frac{n_i}{\sum_i n_i}$) under the given assumption without violating the constraints of the SGO (i.e., $s_a \succ s_b$ requires $p_a^* \geq p_b^*$). 
Since it is also the maximum likelihood possible under each $\mathcal{D}_s$, 
$P(\mathcal{D} | \mathcal{M})$ must also be the maximum for $\mathcal{D}$ under the SGO. 
\end{proof}

In Table \ref{tab:example}, assuming no violation with the given SGO, the $\{p\}$'s
 for $s_1$ and $s_2$ would be set to $\{0, 0.5, 0.5\}$ and $\{1.0, 0.0, 0.0\}$, respectively.  
 However, setting the parameters in this way would violate the given SGO, $s_1 \succ s_2 \succ s_3$.
 An observation is that since the priorities in the SGO are only determined by the maximum generation probabilities for each state, 
we would only need to focus on them to avoid any violations with the SGO. 
Note that it is straightforward to determine the state that should be assigned the maximum generation probability for each state: the state being generated with the largest state count, e.g., $s_3$ or $s_2$ for $s_1$, $s_1$ for $s_2$, and $s_3$ for $s_3$ in Table \ref{tab:example}. 
More generally, it must be that the generation probabilities are assigned to satisfy the order based on the state counts to maximize the likelihood. 

\begin{lem}
Given an SGO, when the maximum generation probability for each state is determined (denoted as $p^*$, respectively, for each state) and these probabilities align with the SGO, the maximum likelihood $P(\mathcal{D} | \mathcal{M})$ is achieved when the remaining generation probabilities for each state are set according to $\min(p^*, \frac{n_j}{\sum_j n_j}(1-\Sigma_k p_k))$ iteratively and in decreasing order of $\frac{n_i}{\sum_i n_i}$, where indices $k$ (and $j$) refer to states that have (and have not) been assigned to a generation probability value, respectively. States being generated with $0$ counts will be assigned equal probability masses if any mass is left unassigned.
\label{lem:remain}
\end{lem}

\begin{proof}
When the maximum generation probability is determined for a state $s$ (denoted by $p^*$), 
the remaining probability entries define an equation that is similar to Eq. $\eqref{eq:product}$, except for that the sum of the remaining $p$ values is $1 - p^*$.
Given the monotonicity of Eq. \eqref{eq:product} from the optimal value setting,
we would need each remaining $p$ value to be as close as possible to $\frac{n_j}{\sum_j n_j}(1-p^*)$
with $p^*$ now assigned to a state. 
When all such values are equal to or smaller than $p^*$, we would be done since the steps above would be equivalent to setting the reamining $p$ values accordingly.

Otherwise, 
first, note that we concluded earlier  that the maximum data likelihood must be achieved by a set of values that respect the order given by $\frac{n_j}{\sum_j n_j}(1-p^*)$. 
Hence, setting the values according to $\min(p^*, \frac{n_j}{\sum_j n_j}(1-\Sigma_k p_k))$ iteratively in the decreasing order is setting the values closest to the optimal setting; otherwise, it is straightforward to see that we can set the $p$ values closer to the optimal value setting to improve the likelihood. 
\end{proof}

For the example in Table \ref{tab:example}, consider a case when the maximum generation probability of $s_1$ is assigned to $s_3$ with a value of $0.4$. In such a case, 
we cannot assign according to $\frac{n_j}{\sum_j n_j}(1-p^*)$ since that would assign all the remaining  probability mass ($0.6$) to $s_2$, leading to $p_2 > p_1$ and a violation with the given SGO.
In such a case, we assign the maximum possible to $s_2$ ($0.4$), and the remaining ($0.2$) to $s_1$.
Note also that assigning $s_2$ to any value smaller than $0.4$ would decrease the likelihood. 

However, problems can occur when the data does not align with the given SGO. 
Without the loss of generality, consider two different states $s_x$ and $s_y$ such that $s_x \succ s_y$ in the given SGO but $p^\mathcal{D}_{x^*} < p^{\mathcal{D}}_{y^*}$, where $p^{\mathcal{D}}_{x^*}$ $(=\frac{n^*}{\sum_i n_i})$ represents the maximum generation probability for state $s_x$ estimated from data under the SGO and $n^*$ denotes the largest count of the state being generated by $s_x$. 

\begin{definition}[Misalignment]
A misalignment occurs when two states $s_x$ and $s_y$ satisfy that $s_x \succ s_y$ in the given SGO but $p^\mathcal{D}_{x^*} < p^{\mathcal{D}}_{y^*}$ based on the data.
\end{definition}

In Table \ref{tab:example}, we can see that there is a misalignment between $s_1$ and $s_2$ ($p^\mathcal{D}_{1^*} = 2/4$ and $p^\mathcal{D}_{2^*} = 2/2$), and a second misalignment between $s_1$ and $s_3$ ($p^\mathcal{D}_{1^*} = 2/4$ and $p^\mathcal{D}_{3^*} = 4/5$), with the given SGO.
Misalignment is directionless so can be represented as an undirected edge between two states. 
In such a representation, all misalignments introduce a connected graph among the states. 
We refer to such a graph as a {\it misalignment graph}.

\begin{lem}
Given an SGO with a misalignment as above, the maximum data likelihood is achieved when $p_{x^*} = p_{y^*}$,
where $p_{x^*}$ ($p_{y^*}$) represents the maximum generation probability of state $s_x$ ($s_y$).
\label{lm:max}
\end{lem}

\begin{proof}
Given $s_x \succ s_y$ in the SGO, it requires that  $p_{x^*} \geq p_{y^*}$.
Given Lemma \ref{lem:maxp}, the 
maximum contributions of $s_x$ and $s_y$ to the data likelihood are achieved at $p_{x^*} = p^\mathcal{D}_{x^*}$ and $p_{y^*} = p^\mathcal{D}_{y^*}$, respectively,  without any constraint from the SGO.  
However, this would imply that $p_{x^*} < p_{y^*}$, leading to a violation with the SGO.
Since the individual contribution is decreasing from its maximum as the values deviate from the optimal value settings for both $s_x$ and $s_y$, 
the maximum combined contribution must be achieved at $p_{x^*} = p_{y^*}$ under the constraint of $p_{x^*} \geq p_{y^*}$.

We prove this above result by contradiction. 
Assume that the maximum is achieved at $p_{x} > p_{y}$ instead. 
In such a case, we can update $p_{x}$ and  $p_{y}$ to share a value in $(p_{y}, p_{x})$ to move both of them closer to the optimal value setting in which $p_{x^*} < p_{y^*}$.
This will increase the likelihood, resulting in a contradiction with the assumption made. 
\end{proof}

\begin{corollary}
For all states that are connected (directly or indirectly) via misalignments in the misalignment graph, the maximum data likelihood is achieved when their maximum generation probabilities are all the same. 
\label{co:same}
\end{corollary}

This is a direct result of Lemma \ref{lm:max}. 
In other words, if $s_x$ and $s_y$, and $s_y$ and $s_z$, are misaligned, they must all share the same maximum generation probabilities to achieve the maximum data likelihood. 
The implication here is that all the states will be divided into connected subgraphs and all states in each connected subgraph must share the same maximum generation probability. 
For Table \ref{tab:example}, given the misalignments, all the states must share the same maximum generation probability under the given SGO. 
In such a case, we can solve for the value of this maximum generation probability by optimizing the following:
\begin{equation}
    \max_p (1-p)^2p^2p^2(1-p)p^4
\label{eq:example}
\end{equation}
\vskip-5pt
which reaches the maximum value at $p = 8/11$.

\begin{theorem}
A (local) maximum data likelihood is achieved by 
assigning the same maximum generation probabilities to states in a subgraph according to $\frac{\sum_c n^*}{\sum_{c} \sum_i n_i}$, where $c$ denotes the states in a connected subgraph in the misalignment graph and $n^*$ denotes the largest count of the state generated by each state in the subgraph. 
The remaining probabilities for each state are assigned according to Lemma \ref{lem:remain}.
\label{th:max}
\end{theorem}

\begin{proof}
The likelihood contributed by a connected subgraph assumes a similar form as in Eq. \eqref{eq:product},
except that the sum of the generation probabilities from individual states must sum to $1$, respectively. 
Given Corollary \ref{co:same}, 
we can formulate the derivative of the data likelihood (similar to that in Eq. \eqref{eq:example}). It is then straightforward to verify that setting the parameters as shown above indeed corresponds to a local maximum.
Conditions on when such a local maximum is also a global maximum will be studied in our future work.
\end{proof}

This theoretical result provides an analytical solution to the parameter estimation problem under a given SGO. 
For the example in Table \ref{tab:example}, we can use this result to derive the maximum generation probability for all three states, $p_c = \frac{2+2+4}{2+2+2+1+4}$, given that they form a connected subgraph.

\subsubsection{Parameter Optimization under Unknown SGO} Next, we extend parameter optimization to unknown SGO. It involves iterating through all possible SGOs.
\begin{theorem}
Iterating through all SGOs $\mathcal{O}$, the MMC, resulted from $o \in \mathcal{O}$ after parameter optimization according to Theorem \ref{th:max}, that maximizes $P(\mathcal{D} | \mathcal{M}_o)$ is the MMC that maximizes the data likelihood among all possible MMC models. 
\end{theorem}

This is a direct result of Theorem \ref{th:max}.
Since the process for determining the optimal set of parameters given the SGO requires only going through the data once, the computational complexity of the learning algorithm is $O(|S|!| \mathcal{D}|)$. 
The complexity is factorial in $|S|$ but only linearly in $\mathcal{D}$.
Also, note that the complexity is indifferent to the order of the MMC,
which allows the optimal solution to be computed for MMCs with small state spaces.
The procedure for MMC parameter optimization is summarized below:

\begin{enumerate}
    \item For each possible $o \in \mathcal{O}$
    \begin{itemize}
    \item Estimate the generation probabilities from data with $o$
    \item Determine the misalignments and subgraphs
    \item Optimize the parameters according to Theorem \ref{th:max}
    \item Compute $P(\mathcal{D} | \mathcal{M}_o)$ with the optimized parameters
     \end{itemize}
    \item Return the SGO achieving the maximum $P(\mathcal{D} | \mathcal{M}_o)$ along with the optimized parameters
\end{enumerate}

\subsection{Towards More Efficient Optimization}
The solution above, however, is intractable
since the number of SGOs can be large when the state space is large. 
In the following, we aim to establish more efficient solutions at the cost of optimality. 

\subsubsection{Hill Climbing} 

The idea is simple. Consider an initial SGO, $o$. 
We aim to incrementally change it to increase the data likelihood.
For each update, we consider swapping the order of a state pair in $o$.
Denote the two states as $s_i$ and $s_j$, with  $x_i$ and $x_j$ as their indices in the current SGO, respectively. 
We consider swapping $s_i$ and $s_j$ while keeping the others unchanged in the SGO. 
This would only require changes to the parameters for states that are in between $s_i$ and $s_j$. 
At any iteration, we select a pair of states to swap that can lead to an improvement until no such state pair can be found after checking all possible pairs. 
The computational complexity for each check is only linear in the data, i.e., $|\mathcal{D}|$, although the number of swaps needed before convergence cannot be predetermined. 
Note that the resulting SGO would be a local maximum. 

\subsubsection{Greedy Heuristic}
A simple heuristic is considered that
orders the states based on 
their maximum generation probabilities estimated from the data. 
At any step, we consider the remaining states as possible candidates for the state of the next highest priority in the SGO.
This step requires going through the data once while ignoring the data entries that are generated by states of higher priorities (i.e., assigned in the previous greedy steps).
The state with the maximum data generation probability among the remaining state is chosen. 
Such a heuristic results in an approximate estimation that runs in $|\mathcal{D}||S|$, allowing MMC to handle large domains.


\subsection{Multiple Generation States}

The above result can be easily extended to define a Max Markov Chain with multiple generation states. 
When there are $N (N \leq K)$ generation states, 
it corresponds to $N!$ generation configurations (permutations of generating positions) for any $N$ states, compared to $1$ with basic MMC.
Note that MMC only considers the relative order between the generation states in the generation configuration (i.e., a state appears before or after another state in the generating positions), not their absolute relative positions.  
Each configuration can thus generate data in $C_K^N$ different generating positions. 
Parameter optimization is mostly unchanged other than that any data entry in the form of $s_{t:t+K-1} \rightarrow s_{t+K}$ may be generated by any one of $C_K^N$ generating positions and the presence of any of them counts as the presence of the corresponding generation configuration. 
The configuration with the maximum generation probability is then chosen as the configuration that generates the next state, similar to the generating state in the MMC.
There is an interesting connection between MMC and HMC: as the number of generation states increases in MMC, MMC would converge gradually to HMC. 
In particular, when $N = K$, MMC is equivalent to a $K$-order Markov chain.

\section{Evaluation}

\begin{figure}[!ht]
  \centering
  \includegraphics[width=1\linewidth]{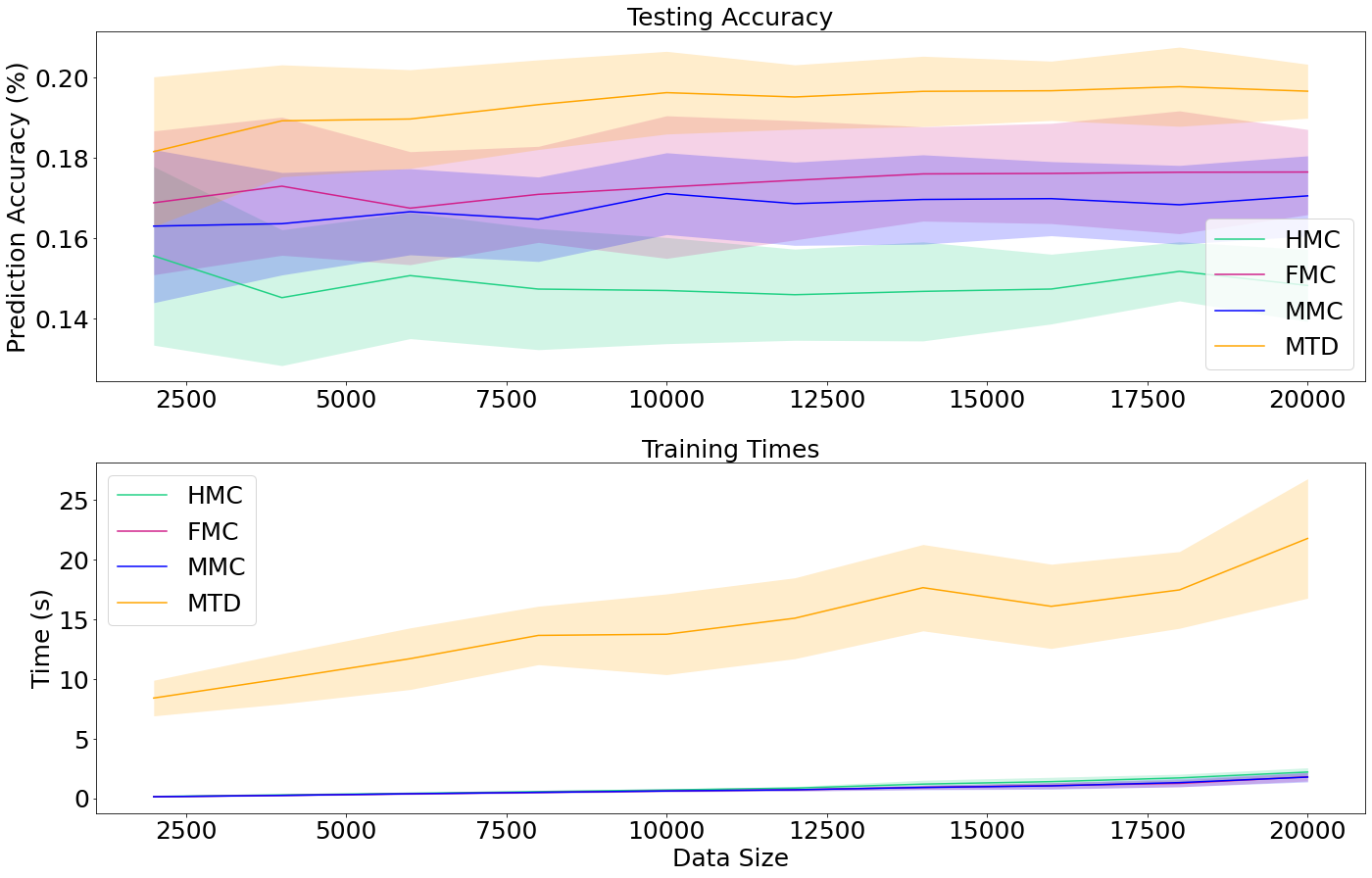}
  \caption{Results for HMC data while varying the data size.  The state size is $7$ and order size is $5$.}
  \label{fig:hmc-data}
  \vskip-10pt
\end{figure}

\begin{figure}[!ht]
  \centering
  \includegraphics[width=1\linewidth]{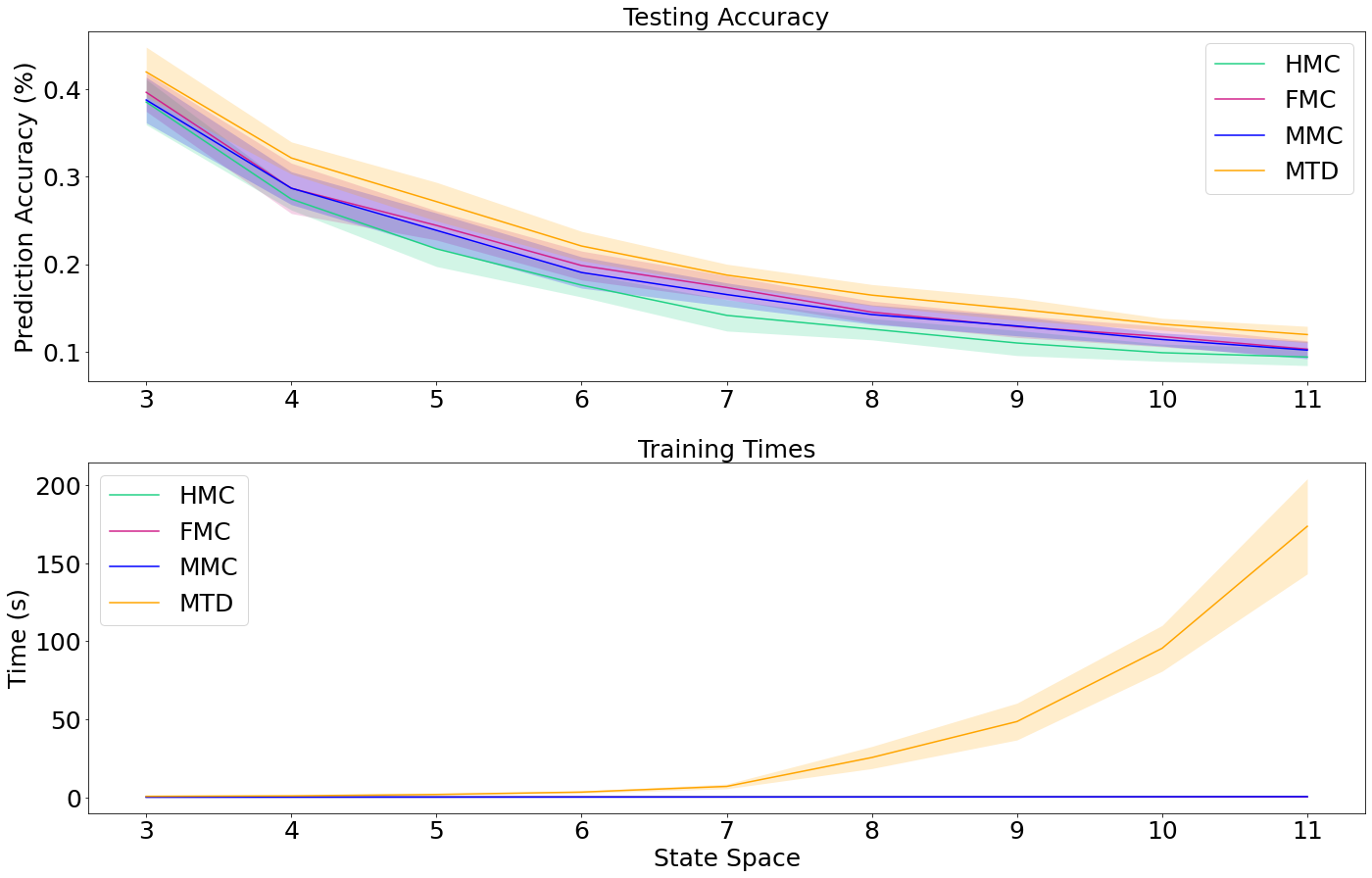}
  \caption{Results for HMC data while varying the state space size.  The data size is 5k and order size is $5$.}
  \label{fig:hmc-state}
  \vskip-10pt
\end{figure}

\begin{figure}[!ht]
  \centering
  \includegraphics[width=1\linewidth]{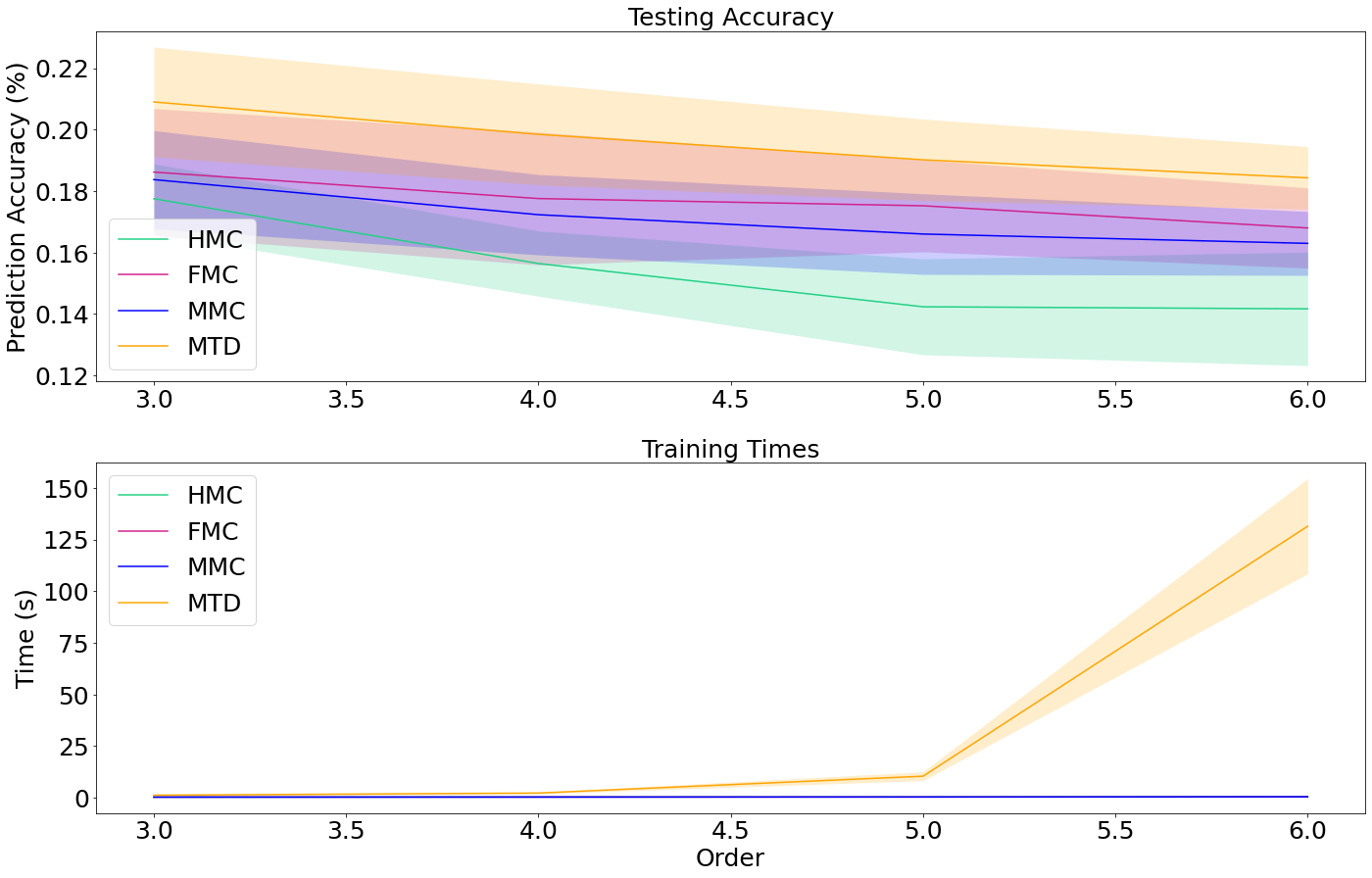}
  \caption{Results for HMC data while varying order size.  The data size is 5k and state size is $7$.}
  \label{fig:hmc-order}
  \vskip-10pt
\end{figure}
In the evaluation, we focus on generating different types of synthetic data to verify the properties of MMC model and illustrate the limitations and advantages of MMC, i.e., when MMC model may be beneficial. 
Extensions of MMC to partial observability, multiple generation states, factored models~\cite{sallans2004reinforcement}, decision models~\cite{russell2010artificial}, etc., will be studied in future work. 
We will also investigate interesting applications of MMC, such as for intention recognition~\cite{sukthankar2014plan} in future work.  
For the purpose of comparison, 
we chose the High-order Markov Model (HMC),
a popular approximate HMC model (MTD~\cite{raftery1985model}), 
and the First-order Markov Model (FMC). 
We compared MMC with these three models with three types of high-order data:
1) HMC data: data generated under the HMC model assumption (Eq. \eqref{eq:hmc-data}); 2) MMC data: data generated under the MMC model assumption (Eq. \eqref{eq:max}); 3) Causal data: data generated under the assumption that the appearance of a state in any previous state or lag is associated with a high probability of generating another specific state. 
All implementations are in Python. 
Since we observe that MMC with the greedy heuristic performed almost equivalently to hill climbing and the analytical solution method, we chose to only show results for MMC with the greedy heuristic.  
Experiments are run on 
Paperspace C7 Instances including 12 vCPUs and 30GB of RAM.  
Each data point is averaged over $30$ runs. 
For each type of data, we test how the different models respond to changes in training data size, state space size, and order size. 

\subsubsection{HMC data}




HMC data is high-order Markov data that satisfies the HMC data generation assumption (Eq. \eqref{eq:hmc-data}). 
Hence, it is expected that MMC would not be able to handle HMC data well. 
The results are presented in Figs. \ref{fig:hmc-data}, 
\ref{fig:hmc-state},
and \ref{fig:hmc-order}. 
First, we can see that HMC model performed poorly even after $20k$ training samples, illustrating its sample inefficiency. 
MTD dominated the others in almost all test cases given its smaller parameter size, which makes it more sample efficient than HMC. 
However, it also used significantly more time than the other models for training. 
MTD appeared to be running linearly with respect to the data size but exponentially in the order and state sizes. 
MMC performed as badly as FMC model (but better than HMC), since each model makes its own assumptions about the data (which do not hold here). 
Even though the MTD model also makes data assumptions, they seem to be milder and the model thus fitted better under randomly generated HMC data.

\subsubsection{MMC data}

The results with MMC data are presented in Figs. \ref{fig:mmc-data}, 
\ref{fig:mmc-state},
and \ref{fig:mmc-order}. 
Results show that MMC indeed outperformed all the other models under MMC data,
which is one type of HMC data. 
What is more interesting was that MTD, again, used significantly more training time than the other models but achieved a performance similar to FMC: it mostly failed to account for the MMC data, which suggested a limitation of MTD models for modeling HMC data. 
Meanwhile, HMC still suffered gravely from sample inefficiency.

\subsubsection{Causal Data}

A common type of data that we may frequently encounter is causal data (see our motivating example). 
Note that causal data only approximately satisfies the MMC data assumptions: it does not restrict that only one lag can generate the current state. 
Hence, it imposes a challenge for MMC. 
The results with causal data are presented in Figs. \ref{fig:causal-data}, 
\ref{fig:causal-state},
and \ref{fig:causal-order}. 
Results show that MMC was able to generalize to this type of data quite well. 
In general, it outperformed the other models on this type of data. 
In Fig. \ref{fig:causal-data}, you can see that MTD model caught up toward the end as more data was provided, which showed that MMC model was more sample efficient than MTD. 
This observation was further confirmed in Fig. \ref{fig:causal-state},
and \ref{fig:causal-order}, 
where MTD started comparable to MMC but failed behind MMC as the state size or order size increased (hence more training data would be needed).

\begin{figure}[!ht]
  \centering
  \includegraphics[width=1\linewidth]{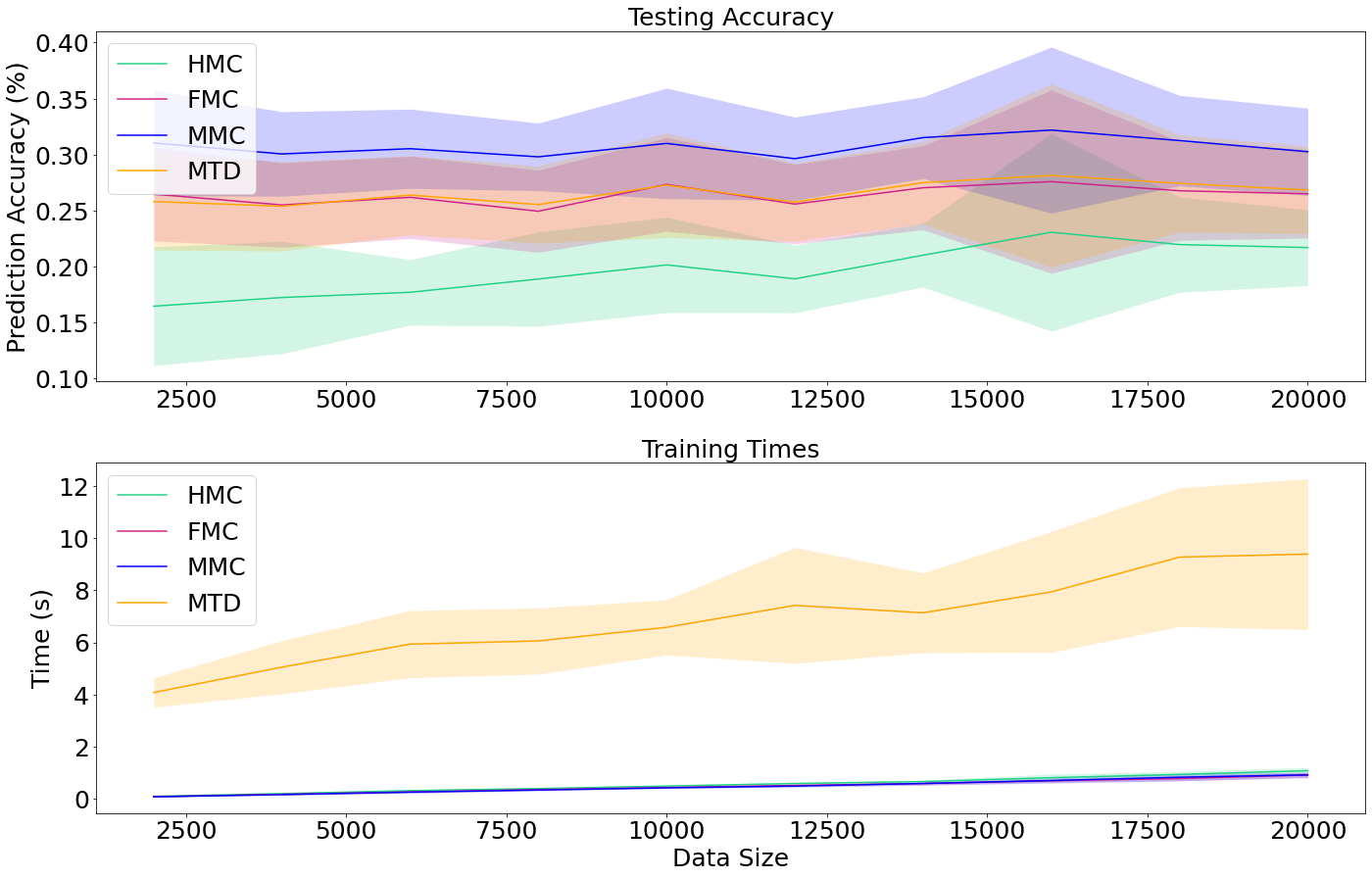}
  \caption{Results for MMC data while varying the data size.  The state size is $7$ and order size is $5$.}
  \label{fig:mmc-data}
  \vskip-10pt
\end{figure}

\begin{figure}[!ht]
  \centering
  \includegraphics[width=1\linewidth]{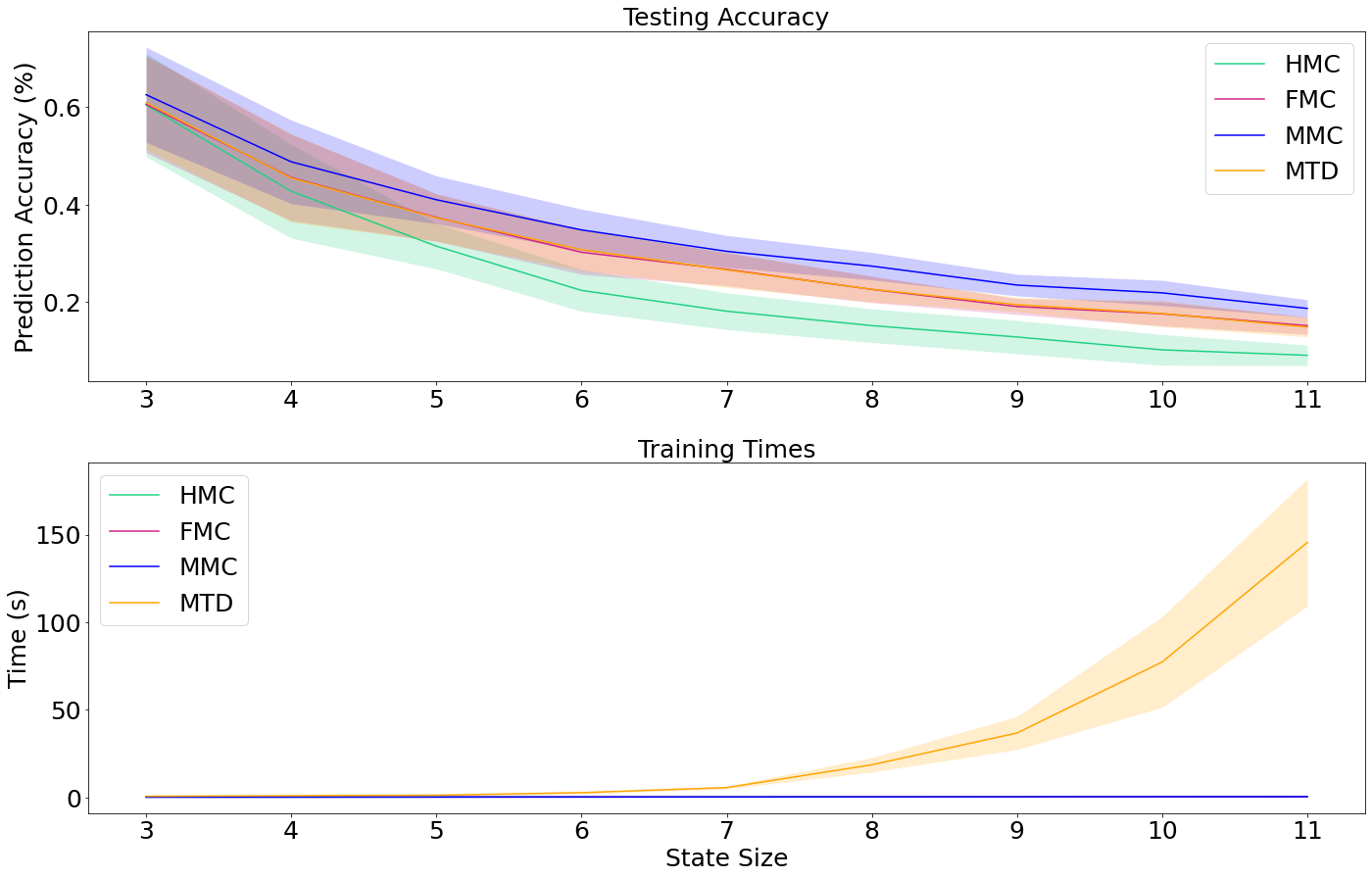}
  \caption{Results for MMC data while varying the state space size.  The data size is $5k$ and order size is $5$.}
  \label{fig:mmc-state}
  \vskip-10pt
\end{figure}

\begin{figure}[!ht]
  \centering
  \includegraphics[width=1\linewidth]{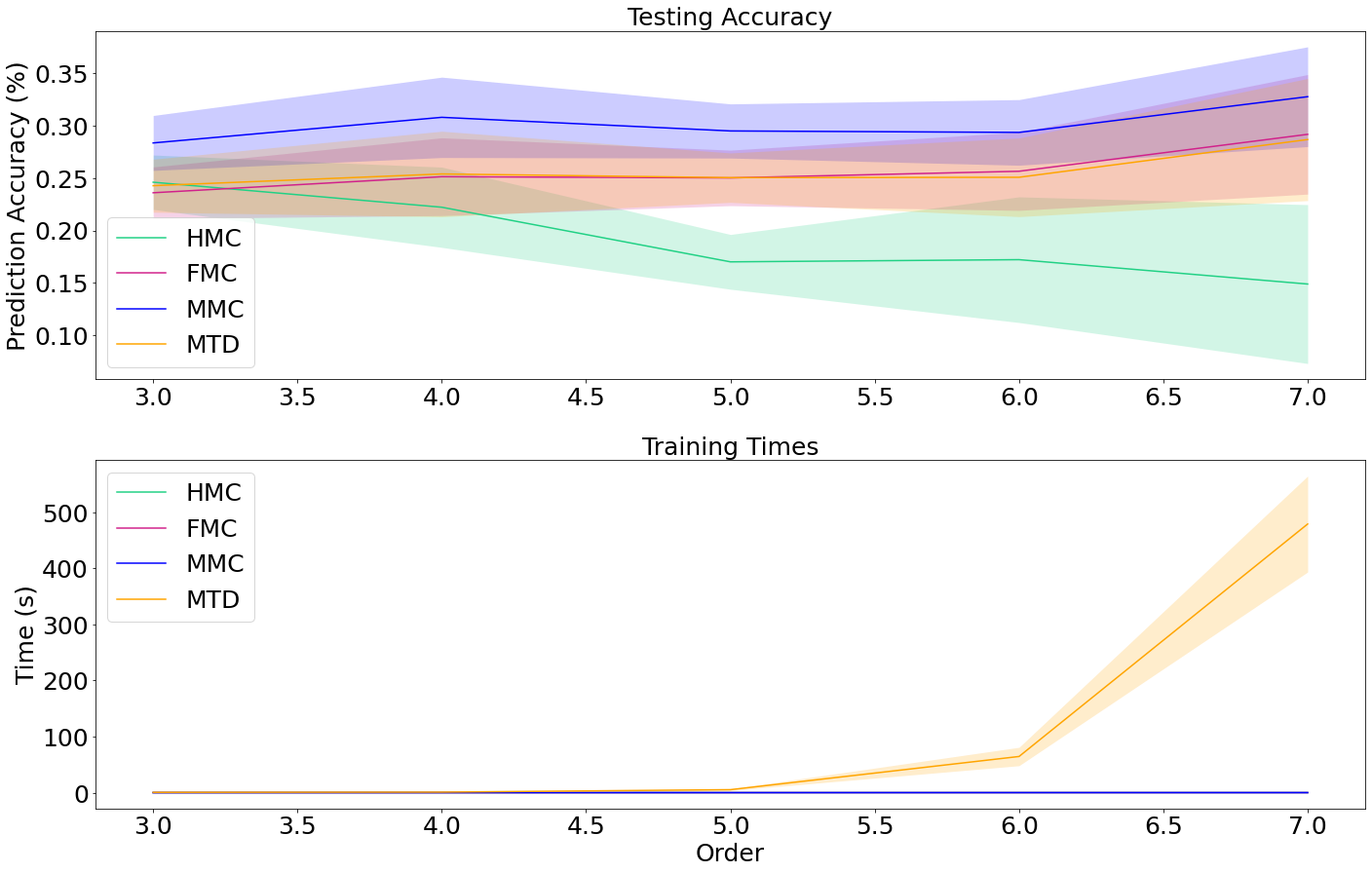}
  \caption{Results for MMC data while varying order size.  The data size is $5k$ and state size is $7$.}
  \label{fig:mmc-order}
  \vskip-10pt
\end{figure}

\begin{figure}[!ht]
  \centering
  \includegraphics[width=1\linewidth]{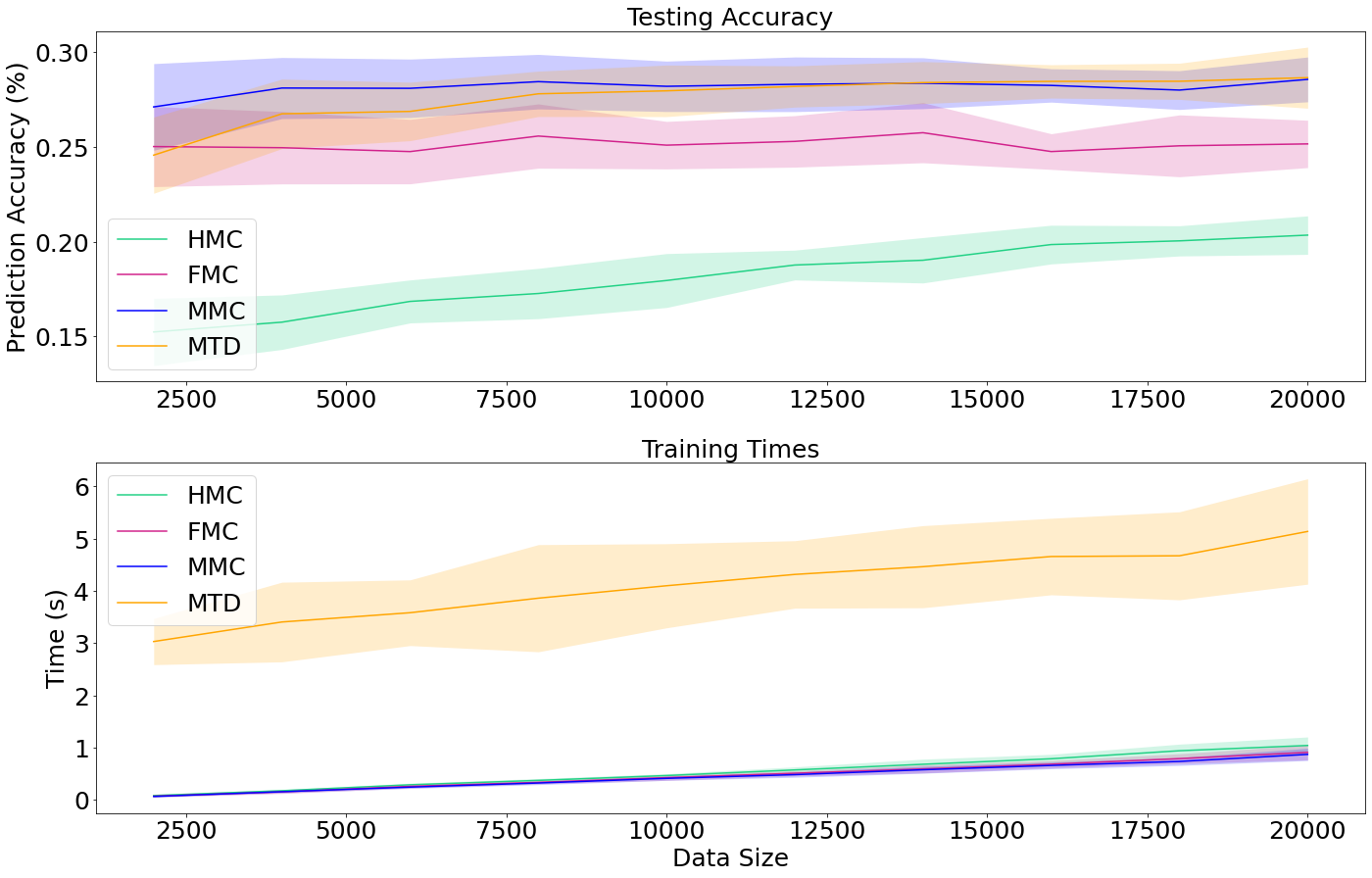}
  \caption{Results for causal data while varying the data size.  The state size is $7$ and order size is $5$.}
  \label{fig:causal-data}
  \vskip-11pt
\end{figure}

\begin{figure}[!ht]
  \centering
  \includegraphics[width=1\linewidth]{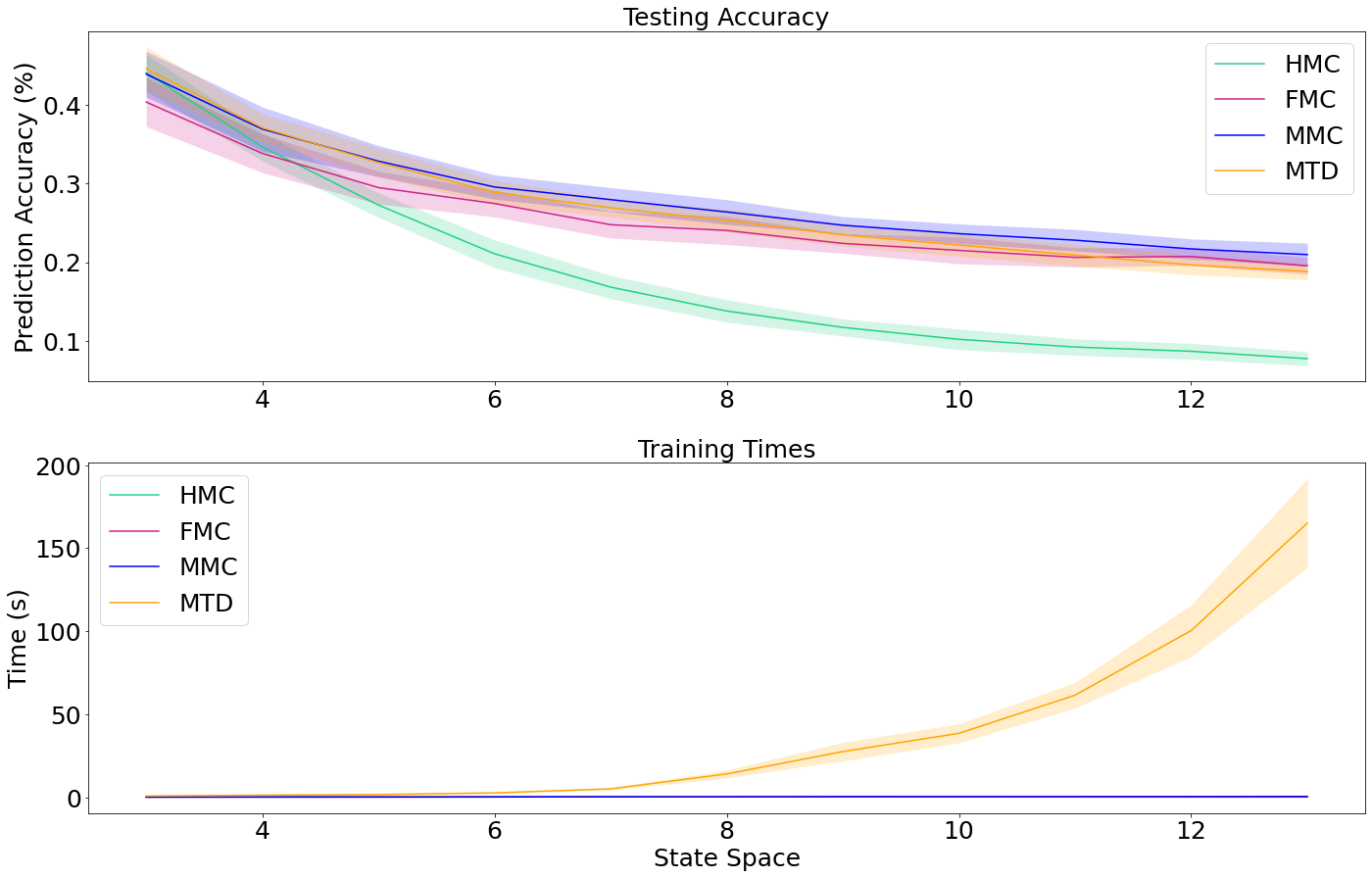}
  \caption{Results for causal data while varying the state space size.  The data size is 5k and order size is $5$.}
  \label{fig:causal-state}
  \vskip-12pt
\end{figure}

\begin{figure}[!ht]
  \centering
  \includegraphics[width=1\linewidth]{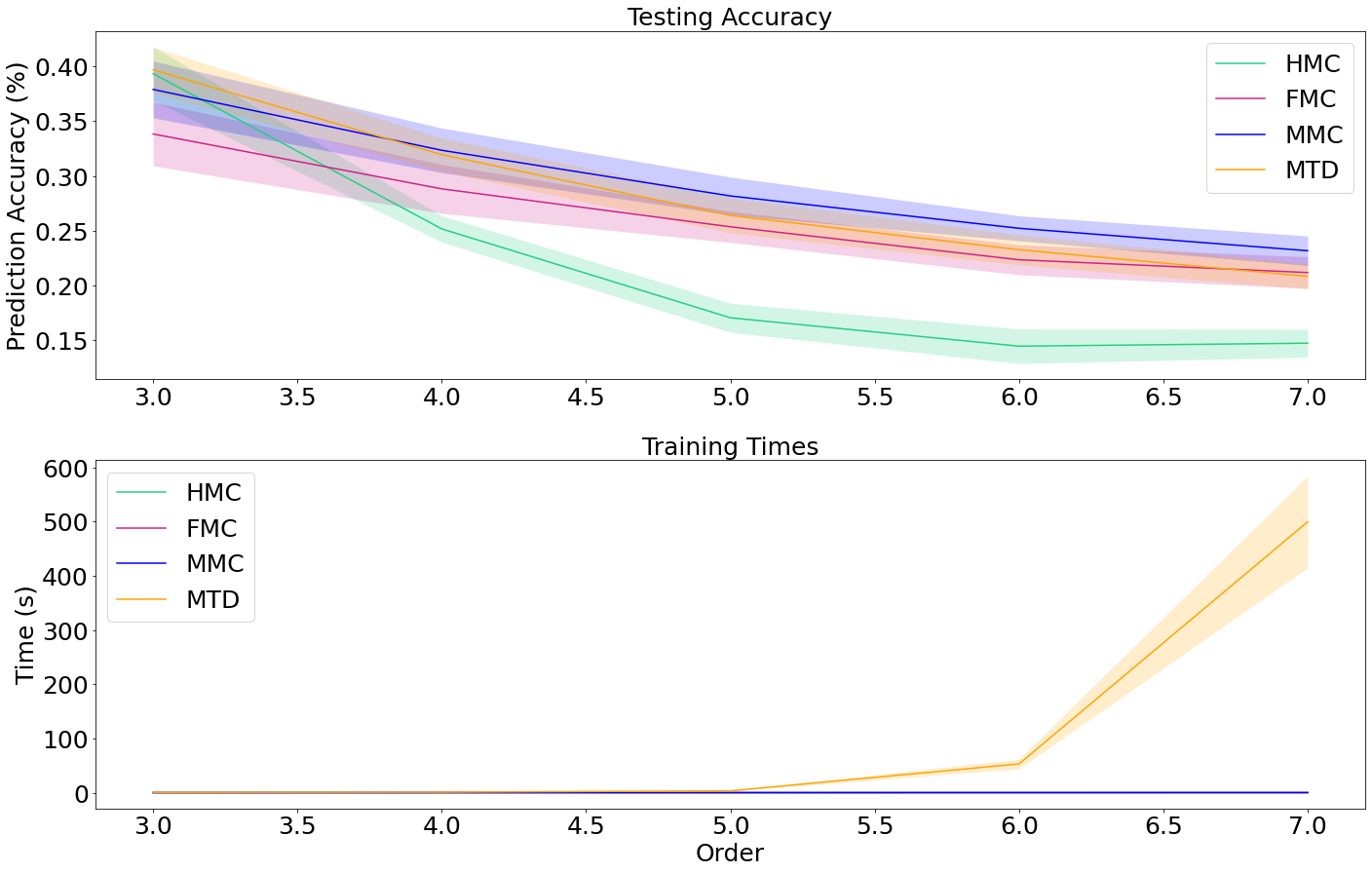}
  \caption{Results for causal data while varying order size.  The data size is 5k and state size is $7$.}
  \label{fig:causal-order}
  \vskip-13pt
\end{figure}

\section{Conclusions}

In this paper, we introduced the Max Markov Chain (MMC) as a novel model for stochastic processes. 
The motivation was to construct an efficient model that enforced parsimony in model structure to model a subset of high-order processes that were useful. 
The simple model structure also enabled the model to scale to large domains. 
We provided an analytical solution for parameter estimation and formally proved it being a local maximum.
Approximate solutions were presented based on hill climbing and a greedy heuristic. 
Results verified that MMC was efficient at handling MMC data, able to generalizing to causal data (which is a common type of data), and scalable to large domains. 
Results also hinted on other domains (e.g., multi-agent task planning and human-robot interaction) where MMC would be expected to excel. 

\bibliography{aaai23}




\end{document}